\documentclass[12pt]{article}

\usepackage[margin=1in]{geometry}
\usepackage{amsthm}
\usepackage{amssymb}
\usepackage{amsmath}
\usepackage{graphicx}
\usepackage{url}
\usepackage{color}
\usepackage{framed}
\usepackage[table]{xcolor}
\usepackage[ruled,vlined]{algorithm2e}

\newtheorem{theorem}{Theorem}

\newtheorem{lemma}[theorem]{Lemma}
\newtheorem{proposition}[theorem]{Proposition}

\theoremstyle{definition}

\title{Lie PCA:\ Density estimation for symmetric manifolds}

\author{Jameson~Cahill\footnote{Department of Mathematics and Statistics, University of North Carolina Wilmington, Wilmington, NC} \and Dustin~G.~Mixon\footnote{Department of Mathematics, The Ohio State University, Columbus, OH} \footnote{Translational Data Analytics Institute, The Ohio State University, Columbus, OH} \and Hans Parshall\footnote{Department of Mathematics, Western Washington University, Bellingham, WA}}
\date{}

\begin{document}
\maketitle

\begin{abstract}
We introduce an extension to local principal component analysis for learning symmetric manifolds.
In particular, we use a spectral method to approximate the Lie algebra corresponding to the symmetry group of the underlying manifold.
We derive the sample complexity of our method for a various manifolds before applying it to various data sets for improved density estimation.
\end{abstract}

\section{Introduction}

Recent advances in machine learning have been made possible by exploiting symmetries and invariants in data.
In 2003, Simard, Steinkraus and Platt~\cite{SimardSP:03} applied two different tricks in this spirit to achieve a record-breaking $0.40$ percent error rate in classifying the MNIST database of handwritten digits.
First, they augmented the training set using the observation that handwritten digits are closed under certain elastic distortions.
Second, they exploited the translation invariance of images by applying a convolutional neural network architecture.
In the time since, both data augmentation and convolutional neural networks have enabled substantial strides in image recognition (e.g., \cite{KrizhevskySH:12,SzegedyEtal:15}).

These engineering feats have inspired various theoretical treatments of symmetries and invariants in data.
Mallat's scattering transform~\cite{Mallat:12,BrunaM:13} provides a principled alternative to convolutional neural networks that exhibits translation invariance and stability to diffeomorphisms.
For settings beyond image classification, other symmetries and invariants must be considered.
In this spirit, Cahill, Contreras and Contreras-Hip~\cite{CahillCC:20,CahillCC:19} identified Lipschitz maps from a signal space $\mathbb{C}^n$ to a low-dimensional feature space in a way that distinguishes orbits in $\mathbb{C}^n$ under the action of a representation of a finite group.
Another approach is to learn symmetries and invariants from the data.
For example, principal component analysis can be viewed as a method of identifying symmetries under the action of a low-dimensional affine group.
For classification tasks, one may seek large linear groups under which the classification is invariant; this approach has been used in~\cite{McWhirterMV:20,DumitrascuVME:19,ClumMS:20}.

In this paper, we consider another fundamental problem in this vein of symmetries and invariants in data.
Suppose you are given the task of augmenting a modest training set.
You are told that there exists a Lie group of deformations (such as elastic distortions) that could be used for this task, but you are not told what the Lie group is.
Can you estimate the Lie group from the data?
In order to measure performance for this task, we phrase the problem in terms of density estimation: Given a sample $\{x_i\}_{i\in[n]}$ from some unknown distribution supported on some unknown symmetric manifold in $\mathbb{R}^d$, produce $\{y_s\}_{s\in[N]}$ with $N\gg n$ that approximates random draws from this unknown distribution.

In the next section, we propose an extension to local principal component analysis for this task.
Our algorithm amounts to a spectral method that estimates the underlying Lie algebra from both the points $\{x_i\}_{i\in[n]}$ and estimates $\{T_i\}_{i\in[n]}$ of the tangent spaces at these points.
In Section~3, we analyze the sample complexity of this approach.
As one would hope, we find that fewer samples are necessary when the manifold has lower dimension and its symmetry group has higher dimension.
We apply our method to the density estimation problem in Section~4, and we conclude in Section~5 with a discussion.

\section{Derivation of Lie PCA}

Given a manifold $M\subseteq\mathbb{R}^d$, denote its \textbf{symmetry group} by
\[
\operatorname{Sym}(M)
:=\{A\in \operatorname{GL}(d): AM=M\}.
\]
We are interested in $M$ for which $\operatorname{Sym}(M)$ is a Lie group and the orbits of $M$ under the action of $\operatorname{Sym}(M)$ are nontrivial.
(See~\cite{Hall:15} for an elementary introduction to matrix Lie groups.)
For example, if $M$ is the unit sphere, then $\operatorname{Sym}(M)=\operatorname{O}(d)$, which is a Lie group.
Also, if $\operatorname{Sym}(M)$ acts transitively on $M$ (e.g., $\operatorname{O}(d)$ acts transitively on the unit sphere), then $M$ is the orbit of any point in $M$.

Let $\mathfrak{sym}(M)\subseteq\mathbb{R}^{d\times d}$ denote the Lie algebra of $\operatorname{Sym}(M)$, that is, the set of all matrices of the form $f'(0)$, where $f$ is a differentiable function from some open interval $0\in I\subseteq\mathbb{R}$ to $\operatorname{Sym}(M)$ such that $f(0)=\operatorname{id}$.
Then the matrix exponential maps $\mathfrak{sym}(M)$ onto the connected component of $\operatorname{Sym}(M)$ that contains $\operatorname{id}$.
Notice that members of $\operatorname{Sym}(M)$ that are close to $\operatorname{id}$ can be realized as $e^A$, where $A\in \mathfrak{sym}(M)$ is close to the zero matrix.
Furthermore, for each $A\in\mathbb{C}^{d\times d}$, it holds that
\[
\|e^{tA}-\operatorname{id}\|_{2\to2}=(1+o(1))\cdot\|A\|_{2\to2}\cdot |t|
\]
as $t\to0$; indeed, it is easy to verify this for diagonalizable matrices, which are dense in the set of complex matrices.
As such, if we know $\mathfrak{sym}(M)$, then for every $x\in M$, it holds that $e^Ax$ is a slight perturbation of $x$ in $M$ for every small $A\in\mathfrak{sym}(M)$.
This is the heart of our approach to the density estimation problem, but in order for this to work, we first need to estimate $\mathfrak{sym}(M)$ from a sample $\{x_i\}_{i\in[n]}$ of $M$.
We accomplish this in two steps:
\begin{itemize}
\item[1.]
Use $\{x_i\}_{i\in[n]}$ to obtain an estimate $\{T_i\}_{i\in[n]}$ of the tangent spaces $\{T_{x_i}M\}_{i\in[n]}$.
\item[2.]
Use $\{x_i\}_{i\in[n]}$ and $\{T_i\}_{i\in[n]}$ to obtain an estimate of $\mathfrak{sym}(M)$.
\end{itemize}
The first step above is well understood: local PCA.
That is, for each $i\in[n]$, we select the $x_j$'s closest to $x_i$ and run principal component analysis (PCA) on this subcollection to estimate $T_i$.
For the second step, we apply Algorithm~\ref{alg.LiePCA}, which we derive in this section.
Our approach is motivated by the following observation:

\begin{algorithm}[t]
\SetAlgoLined
\KwData{Sample $\{x_{i}\}_{i\in[n]}$ of $M\subseteq\mathbb{R}^{d}$, estimate $\{T_i\}_{i\in[n]}$ of $\{T_{x_i}M\}_{i\in[n]}$, and $\ell\in\mathbb{N}$}
\KwResult{Estimate of Lie algebra $\mathfrak{sym}(M)$}
Define $\Sigma\colon\mathbb{R}^{d\times d}\to\mathbb{R}^{d\times d}$ by $\Sigma(A):=\sum_{i\in[n]}\operatorname{proj}_{T_i^\perp}\cdot A\cdot\operatorname{proj}_{\operatorname{span}\{x_i\}}$.\\
Compute eigenvectors $\{v_j\}_{j\in[\ell]}$ of $\Sigma$ corresponding to the $\ell$ smallest eigenvalues.\\
Output $\operatorname{span}\{v_j\}_{j\in[\ell]}$.
 \caption{Lie principal component analysis
 \label{alg.LiePCA}}
\end{algorithm}

\begin{lemma}
\label{lem.lie alg tangent}
For every $x\in M$ and $A\in\mathfrak{sym}(M)$, it holds that $Ax\in T_xM$.
\end{lemma}

\begin{proof}
Fix $x\in M$ and $A\in\mathfrak{sym}(M)$.
Consider any differentiable function $f\colon I\to\operatorname{Sym}(M)$ such that $f(0)=\operatorname{id}$ and $f'(0)=A$.
Then $g\colon I\to M$ defined by $g(t):=f(t)x$ is differentiable with $g(0)=f(0)x=x$, and so $Ax=f'(0)x=g'(0)\in T_xM$.
\end{proof}

For each $x\in M$, denote
\[
S_xM:=\{A\in\mathbb{R}^{d\times d}:Ax\in T_xM\}.
\]
By Lemma~\ref{lem.lie alg tangent}, we have $S_xM\supseteq \mathfrak{sym}(M)$ for every $x\in M$.
Given a sample $\{x_i\}_{i\in[n]}$ in $M$ and corresponding tangent spaces $\{T_{x_i}M\}_{i\in[n]}$, we may estimate $\mathfrak{sym}(M)$ by $\bigcap_{i\in[n]} S_{x_i}M$.
We seek a numerically robust version of this estimate.
To this end, it is convenient to write
\begin{equation}
\label{eq.approx intersection}
\bigcap_{i\in[n]} S_{x_i}M
=\bigg( \sum_{i\in[n]} (S_{x_i}M)^\perp \bigg)^\perp
=\operatorname{ker}\bigg( \sum_{i\in[n]} \operatorname{proj}_{(S_{x_i}M)^\perp} \bigg).
\end{equation}
The terms in this sum have a convenient expression:

\begin{lemma}
\label{lem.projSxM}
For every $x\in M\setminus\{0\}$, it holds that
$\operatorname{proj}_{(S_{x}M)^\perp}(A)
=\operatorname{proj}_{N_xM}\cdot A\cdot\operatorname{proj}_{\operatorname{span}\{x\}}$.
\end{lemma}

Our proof of this lemma makes use of the following expression for $(S_{x}M)^\perp$:

\begin{lemma}
\label{lem.normal space}
For every $x\in M$, it holds that
\begin{itemize}
\item[(a)]
$S_xM=\{yx^\top+B:y\in T_xM,~B\in\mathbb{R}^{d\times d},~Bx=0\}$, and
\item[(b)]
$(S_xM)^\perp=\{zx^\top:z\in N_xM\}$.
\end{itemize}
\end{lemma}

\begin{proof}
In the case where $x=0$, we have $S_xM=\mathbb{R}^{d\times d}$ and both (a) and (b) are immediate.
It remains to consider the case in which $x\neq0$.
(a)
First, ($\supseteq$) follows from the fact that $(yx^\top+B)x=\|x\|^2y\in T_xM$.
For ($\subseteq$), given $A\in S_xM$, take $y=\|x\|^{-2}Ax$ and $B=A(I-\|x\|^{-2}xx^\top)$ so that $A=yx^\top+B$ with $y\in T_xM$ and $Bx=0$.
(b)
First, ($\supseteq$) follows from the fact that
\[
\langle zx^\top,yx^\top+B\rangle
=\operatorname{tr}((zx^\top)^\top(yx^\top+B))
=\operatorname{tr}(xz^\top(yx^\top+B))
=\|x\|^2z^\top y+z^\top Bx
=0.
\]
For ($\subseteq$), suppose $Z\in(S_xM)^\perp$.
For every $u\in\mathbb{R}^d$ and $v\in\operatorname{span}\{x\}^\perp$, since $B=uv^\top\in S_xM$ satisfies $Bx=0$, we have $0=\langle Z,uv^\top\rangle=\operatorname{tr}(Z^\top uv^\top)=v^\top Z^\top u=\langle Zv,u\rangle$.
Since $u\in\mathbb{R}^d$ is arbitrary, it follows that $Zv=0$, and since $v\in\operatorname{span}\{x\}^\perp$ is arbitrary, we conclude that $Z=zx^\top$ for some $z\in\mathbb{R}^d$.
Next, for every $y\in T_xM$, since $yx^\top\in S_xM$, it holds that $0=\langle Z,yx^\top\rangle=\operatorname{tr}(Z^\top yx^\top)=\|x\|^2\langle z,y\rangle$, and so $z\in (T_xM)^\perp=N_xM$, as desired.
\end{proof}

\begin{proof}[Proof of Lemma~\ref{lem.projSxM}]
Fix $A\in\mathbb{R}^{d\times d}$.
We seek the member of $(S_xM)^\perp$ that is closest to $A$ in Frobenius norm.
By Lemma~\ref{lem.normal space}(b), every member of $(S_xM)^\perp$ takes the form $zx^\top$ for some $z\in N_xM$.
Letting $P$ denote orthogonal projection onto $N_xM$, then Cauchy--Schwarz gives
\begin{align*}
\|A-zx^\top\|_F^2
&=\|A\|_F^2-2\langle A,zx^\top\rangle+\|zx^\top\|_F^2\\
&=\|A\|_F^2-2\langle Ax,z\rangle+\|z\|_2^2\|x\|_2^2\\
&=\|A\|_F^2-2\langle PAx,z\rangle+\|z\|_2^2\|x\|_2^2\\
&\geq\|A\|_F^2-2\|PAx\|_2\|z\|_2+\|z\|_2^2\|x\|_2^2
\geq\min_{t\geq0}(\|A\|_F^2-2\|PAx\|_2\cdot t+\|x\|_2^2\cdot t^2).
\end{align*}
Equality is achieved in the first inequality precisely when $z$ is a positive multiple of $PAx$.
For the second inequality, equality is achieved precisely when $\|z\|_2=\|PAx\|_2/\|x\|_2^2$.
Overall,
\[
\operatorname{proj}_{(S_{x}M)^\perp}(A)
=\Big(\frac{PAx}{\|PAx\|_2}\cdot\frac{\|PAx\|_2}{\|x\|_2^2}\Big)x^\top
=PA\frac{xx^\top}{\|x\|_2^2}
=\operatorname{proj}_{N_xM}\cdot A\cdot\operatorname{proj}_{\operatorname{span}\{x\}}.
\qedhere
\]
\end{proof}

We are now ready to derive our approach.
Let $\ell\in\mathbb{N}$ denote the dimension of the desired Lie algebra.
(For our algorithm, this quantity is treated as a hyperparameter.)
Next, let $L(\ell,d)$ denote the set of all $\ell$-dimensional Lie algebras of Lie subgroups of $\operatorname{GL}(d)$.
By ignoring the multiplicative structure, we may think of $L(\ell,d)$ as a subset of the Grassmannian $\operatorname{Gr}(\ell,\mathbb{R}^{d\times d})$.
Given a sample $\{x_i\}_{i\in[n]}$ in $M\subseteq\mathbb{R}^d$ and a noisy estimate $\{T_i\}_{i\in[n]}$ of $\{T_{x_i}M\}_{i\in[n]}$, we define $\Sigma\colon\mathbb{R}^{d\times d}\to\mathbb{R}^{d\times d}$ by 
\[
\Sigma(A)
:=\sum_{i\in[n]}\operatorname{proj}_{T_i^\perp}\cdot A\cdot\operatorname{proj}_{\operatorname{span}\{x_i\}}.
\]
Note that by Lemma~\ref{lem.projSxM}, the $i$th term approximates $\operatorname{proj}_{(S_{x_i}M)^\perp}$.
Considering \eqref{eq.approx intersection}, we are compelled to solve the program
\[
\text{minimize}
\quad
\langle \Sigma,\operatorname{proj}_{\mathfrak{a}}\rangle
\quad
\text{subject to}
\quad
\mathfrak{a}\in L(\ell,d).
\]
Since optimizing over $L(\ell,d)$ is cumbersome, we relax to the entire Grassmannian:
\[
\text{minimize}
\quad
\langle \Sigma,X\rangle
\quad
\text{subject to}
\quad
X^2=X,
\quad
X^*=X,
\quad
\operatorname{rank}X=\ell.
\]
As a consequence of the Poincar\'{e} separation theorem, the orthogonal projection onto the span of any $\ell$ of the bottom eigenvectors of $\Sigma$ gives an optimizer of this program.
In particular, the span of any such eigenvectors gives a worthy estimate for $\mathfrak{sym}(M)$.
One may be inclined to round this estimate to the nearest member of $L(\ell,d)$, but we do not attempt this here.

\section{Sample complexity of Lie PCA}

In this section, we consider an idealized setting in which the estimate $\{T_i\}_{i\in[n]}$ of $\{T_{x_i}M\}_{i\in[n]}$ is exact.
Intuitively, Lie PCA should require fewer samples when $M$ is low-dimensional and $\mathfrak{sym}(M)$ is high-dimensional.
This intuition matches the following lower bound on the sample complexity of Lie PCA:

\begin{lemma}
\label{lem.sample complexity}
It holds that $\mathfrak{sym}(M)\subseteq\bigcap_{i\in[n]} S_{x_i}M$, with equality only if
\[
n
\geq
n^\star(M)
:=\frac{\operatorname{codim}\mathfrak{sym}(M)}{\operatorname{codim} M}.
\]
Furthermore, $\mathfrak{sym}(M)=\bigcap_{i\in[n^\star(M)]} S_{x_i}M$ precisely when the subspaces $\{(S_{x_i}M)^\perp\}_{i\in[n^\star(M)]}$ are linearly independent.
\end{lemma}

\begin{proof}
Lemma~\ref{lem.lie alg tangent} gives that $\mathfrak{sym}(M)\subseteq S_{x_i}M$ for every $i\in[n]$, and so the desired inclusion holds.
Next, suppose equality holds.
Then $\mathfrak{sym}(M)^\perp=\sum_{i\in[n]}(S_{x_i}M)^\perp$, and so
\[
\operatorname{codim}\mathfrak{sym}(M)
=\operatorname{dim}\mathfrak{sym}(M)^\perp
\leq \sum_{i\in[n]}\operatorname{dim}(S_{x_i}M)^\perp
\leq \sum_{i\in[n]}\operatorname{dim}(N_{x_i}M)
=n\cdot\operatorname{codim}M,
\]
where the second inequality follows from Lemma~\ref{lem.normal space}(b).
Rearranging gives the desired bound, and the last statement follows from counting dimensions.
\end{proof}

We observe that in many cases, the bound in Lemma~\ref{lem.sample complexity} is the threshold at which generic samples determine the Lie algebra.
To make this rigorous, let $\mathcal{O}(M^n)$ denote the set of subsets of $M^n\subseteq(\mathbb{R}^d)^n$ that are open and dense in the subspace topology, and define
\[
n^\circ(M)
:=\inf\Big\{n\in\mathbb{N}:\exists O\in\mathcal{O}(M^n),~\forall \{x_i\}_{i\in[n]}\in O,~\mathfrak{sym}(M)=\bigcap_{i\in[n]} S_{x_i}M\Big\}.
\]
(As a mnemonic, think of $\circ$ as the ``o'' in ``open.'')
Lemma~\ref{lem.sample complexity} implies that $n^\circ(M)\geq n^\star(M)$ for every $M$, and in this section, we show that $n^\circ(M)= n^\star(M)$ for several choices of $M$.
We will continually make use of the following:

\begin{lemma}
\label{lem.gl wlog}
For any manifold $M\subseteq\mathbb{R}^d$ and any $Z\in\operatorname{GL}(d)$, it holds that
\[
n^\star(ZM)=n^\star(M),
\qquad
n^\circ(ZM)=n^\circ(M).
\]
\end{lemma}

\begin{proof}
First, we show that that $\operatorname{Sym}(ZM)=Z\cdot\operatorname{Sym}(M)\cdot Z^{-1}$.
Suppose $A\in\operatorname{Sym}(M)$.
Then $A\in\operatorname{GL}(d)$ such that $AM=M$.
Then $ZAZ^{-1}ZM=ZAM=ZM$, meaning $ZAZ^{-1}\in\operatorname{Sym}(ZM)$.
Similarly, $B\in\operatorname{Sym}(ZM)$ implies $Z^{-1}BZ\in\operatorname{Sym}(M)$, and the claim follows.

Next, we show that that $\mathfrak{sym}(ZM)=Z\cdot\mathfrak{sym}(M)\cdot Z^{-1}$.
Take any continuously differentiable $f\colon I\to\operatorname{Sym}(ZM)$ with $f(0)=\operatorname{id}$.
Then $g\colon I\to\operatorname{Sym}(M)$ defined by $g(t):=Z^{-1}f(t)Z$ satisfies $g(0)=\operatorname{id}$ and $g'(0)=Z^{-1}f'(0)Z$.
This implies $\mathfrak{sym}(ZM)\subseteq Z\cdot\mathfrak{sym}(M)\cdot Z^{-1}$, and the reverse containment holds by a similar argument.

A similar argument demonstrates that $T_{Zx}ZM = Z\cdot T_xM$.
As such, $\operatorname{dim}\mathfrak{sym}(ZM)=\operatorname{dim}\mathfrak{sym}(M)$ and $\operatorname{dim}ZM=\operatorname{dim}M$, from which it follows that $n^\star(ZM)=n^\star(M)$.

Now we verify that $n^\circ(ZM)=n^\circ(M)$.
First, since $\langle Zx,y\rangle=\langle x,Z^\top y\rangle$, it follows that $z$ is orthogonal to $x$ precisely when $(Z^\top)^{-1}z$ is orthogonal to $Zx$.
As such,
\[
N_{Zx}ZM
=(T_{Zx}ZM)^\perp
=(Z\cdot T_xM)^\perp
=(Z^\top)^{-1}\cdot (T_xM)^\perp
=(Z^\top)^{-1}\cdot N_xM.
\]
Combining this with Lemma~\ref{lem.normal space}(b) then gives
\begin{align*}
(S_{Zx}ZM)^\perp
&=\{z(Zx)^\top:z\in N_{Zx}ZM\}\\
&=\{(Z^\top)^{-1} ux^\top Z^\top:u\in N_xM\}
=(Z^\top)^{-1}\cdot(S_{x}M)^\perp\cdot Z^\top.
\end{align*}
Next, since $\langle (Z^\top)^{-1}XZ^\top,Y\rangle=\langle X ,Z^{-1}YZ\rangle$, we may similarly conclude that
\begin{align*}
S_{Zx}ZM
=((S_{Zx}ZM)^\perp)^\perp
&=((Z^\top)^{-1}\cdot(S_{x}M)^\perp\cdot Z^\top)^\perp\\
&=Z\cdot ((S_{x}M)^\perp)^\perp\cdot Z^{-1}
=Z\cdot S_{x}M \cdot Z^{-1}.
\end{align*}
As such, $\mathfrak{sym}(M)=\bigcap_{i\in[n]} S_{x_i}M$ precisely when
\[
\mathfrak{sym}(ZM)
=Z\cdot\mathfrak{sym}(M)\cdot Z^{-1}
=Z\cdot\bigg(\bigcap_{i\in[n]} S_{x_i}M\bigg)\cdot Z^{-1}
=\bigcap_{i\in[n]} S_{Zx_i}ZM.
\]
Finally, since the bicontinuous mapping $f\colon\{x_i\}_{i\in[n]}\to\{Zx_i\}_{i\in[n]}$ has the property that $f(\mathcal{O}(M^n))=\mathcal{O}((ZM)^n)$, the claim follows.
\end{proof}

We start by treating the case in which $M$ is a subspace:

\begin{theorem}
\label{thm.subspace}
Let $M$ be any $r$-dimensional subspace of $\mathbb{R}^d$.
Then $n^\circ(M)=n^\star(M)=r$.
\end{theorem}

\begin{proof}
By Lemma~\ref{lem.gl wlog}, we have $M=\operatorname{span}\{e_i\}_{i\in[r]}$ without loss of generality.
We claim that
\[
\operatorname{Sym}(M)
=\{[\begin{smallmatrix}A&B\\0&D\end{smallmatrix}]:A\in\operatorname{GL}(r),D\in\operatorname{GL}(d-r)\}.
\]
To see this, suppose $[\begin{smallmatrix}A&B\\C&D\end{smallmatrix}]\in\operatorname{GL}(d)$ with $A\in\mathbb{R}^{r\times r}$ satisfies $[\begin{smallmatrix}A&B\\C&D\end{smallmatrix}]M=M$.
Then $C=0$, which forces $A\in\operatorname{GL}(r)$ and $D\in\operatorname{GL}(d-r)$.
Conversely, every $[\begin{smallmatrix}A&B\\C&D\end{smallmatrix}]$ of this form satisfies $[\begin{smallmatrix}A&B\\C&D\end{smallmatrix}]M=M$.

Next, we claim that
\[
\mathfrak{sym}(M)
=\{[\begin{smallmatrix}A&B\\0&D\end{smallmatrix}]:A\in\mathbb{R}^{r\times r},D\in\mathbb{R}^{(d-r)\times(d-r)}\}.
\]
Let $S$ denote the right-hand side.
For ($\subseteq$), select $Z\in\mathfrak{sym}(M)$ and consider any continuously differentiable $f\colon I\to\operatorname{Sym}(M)$ with $f(0)=\operatorname{id}$ and $f'(0)=Z$.
Then $Z=\lim_{h\to0}\frac{1}{h}(f(h)-\operatorname{id})$.
Considering $\frac{1}{h}(f(h)-\operatorname{id})\in S$ for every $h\in I$ and $S$ is closed, it follows that $Z\in S$.
For ($\supseteq$), select $Z\in S$ and pick a small enough interval $I$ about $0$ so that $\operatorname{det}(\operatorname{id}+tZ)\neq0$ for every $t\in I$.
Then $f(t):=\operatorname{id}+tZ$ is a continuously differentiable map from $I$ to $\operatorname{Sym}(M)$.
Furthermore, $f(0)=\operatorname{id}$ and $f'(t)=Z$, and so $Z\in\mathfrak{sym}(M)$.

At this point, we may compute
\[
n^\star(M)
=\frac{\operatorname{codim}\mathfrak{sym}(M)}{\operatorname{codim}M}
=\frac{r(d-r)}{d-r}
=r.
\]
Furthermore, every $x\in M$ takes the form $x=[\begin{smallmatrix}v\\0\end{smallmatrix}]$ with $v\in\mathbb{R}^r$, and so Lemma~\ref{lem.normal space}(b) gives
\[
(S_xM)^\perp
=\{zx^\top:z\in N_xM\}
=\{[\begin{smallmatrix}0&0\\uv^\top&0\end{smallmatrix}]:u\in\mathbb{R}^{d-r}\}.
\]
We claim that if the vectors $\{x_i\}_{i\in[n]}$ are linearly independent, then the subspaces $\{(S_{x_i}M)^\perp\}_{i\in[n]}$ are linearly independent.
To see this, suppose $\{(S_{x_i}M)^\perp\}_{i\in[n]}$ are dependent.
Then there exist scalars $\{\alpha_{ij}\}_{i\in[n],j\in[r]}$, not all of which are zero, such that
\[
\sum_{i\in[n]}\sum_{j\in[d-r]}\alpha_{ij} e_{r+j}x_i^\top
=0.
\]
Select $p\in[n]$ and $q\in[d-r]$ such that $\alpha_{pq}\neq0$.
Then
\[
\sum_{i\in[n]}\alpha_{iq}x_i^\top
=e_q^\top\bigg(\sum_{i\in[n]}\sum_{j\in[d-r]}\alpha_{ij} e_{r+j}x_i^\top\bigg)
=0,
\]
i.e., the vectors $\{x_i\}_{i\in[n]}$ are dependent, as claimed.
Considering Lemma~\ref{lem.sample complexity}, the result follows by taking $O$ to be the set of linearly independent $\{x_i\}_{i\in[r]}\in M^r$.
\end{proof}

Next, we consider the case in which $M$ is a strictly affine subspace, that is, any translation of a subspace that is not itself a subspace:

\begin{theorem}
\label{thm.affine subspace}
Let $M$ be any $r$-dimensional strictly affine subspace of $\mathbb{R}^d$.
Then $n^\circ(M)=n^\star(M)=r+1$.
\end{theorem}

\begin{proof}[Proof sketch]
The proof is nearly identical to that of Theorem~\ref{thm.subspace}, so we state the intermediate claims without providing details.
By Lemma~\ref{lem.gl wlog}, we have $M=\operatorname{span}\{e_i\}_{i\in[r]}+e_{r+1}$ without loss of generality.
Then
\begin{align*}
\operatorname{Sym}(M)
&=\{[\begin{smallmatrix}A&B\\0&D\end{smallmatrix}]:A\in\operatorname{GL}(r),D\in\operatorname{GL}(d-r),De_1=e_1\},\\
\mathfrak{sym}(M)
&=\{[\begin{smallmatrix}A&B\\0&D\end{smallmatrix}]:A\in\mathbb{R}^{r\times r},D\in\mathbb{R}^{(d-r)\times(d-r)},De_1=0\}.
\end{align*}
It follows that
\[
n^\star(M)
=\frac{\operatorname{codim}\mathfrak{sym}(M)}{\operatorname{codim}M}
=\frac{(r+1)(d-r)}{d-r}
=r+1.
\]
Furthermore, if $\{x_i\}_{i\in[n]}$ are independent, then $\{(S_{x_i}M)^\perp\}_{i\in[n]}$ are independent.
The result follows by taking $O$ to be the set of linearly independent $\{x_i\}_{i\in[r+1]}\in M^{r+1}$.
\end{proof}

Next, we consider certain types of quadrics.
The first type of quadric we consider includes spheres and hyperboloids:

\begin{theorem}
\label{thm.quadric}
Select any full rank $Q\in \mathbb{R}_\mathrm{sym}^{d\times d}$ and put $M:=\{x\in\mathbb{R}^d:x^\top Qx=1\}$.
\begin{itemize}
\item[(a)]
$M$ is empty if and only if $Q$ is negative definite.
\item[(b)]
Otherwise, $n^\circ(M)=n^\star(M)=\binom{d+1}{2}$.
\end{itemize}
\end{theorem}

We will apply the following well-known consequence of the implicit function theorem:

\begin{lemma}
\label{lem.normal gradient}
Select any continuously differentiable $f\colon\mathbb{R}^d\to\mathbb{R}$ and put
\[
M:=\{x\in\mathbb{R}^d:f(x)=0,\nabla f(x)\neq0\}.
\]
Then $N_xM=\operatorname{span}\{\nabla f(x)\}$.
\end{lemma}

\begin{proof}[Proof of Theorem~\ref{thm.quadric}]
First, (a) is immediate.
For (b), take the eigenvalue decomposition $Q=U\Lambda U^\top$, and we decompose $\Lambda=T^{1/2}ST^{1/2}$ with $T=\operatorname{abs}(\Lambda)$ and $S=\operatorname{sgn}(\Lambda)$, where these operations are performed entrywise.
Notice that $Q$ being full rank implies that $U$ and $T$ are both full rank.
Substituting $x=T^{1/2}U^\top z$ gives
\begin{align*}
M
&:=\{x\in\mathbb{R}^d:x^\top Qx=1\}\\
&=\{x\in\mathbb{R}^d:x^\top UT^{1/2}ST^{1/2} U^\top x=1\}
=(T^{1/2} U^\top)^{-1}\cdot \{z\in\mathbb{R}^d: z^\top Sz=1\}.
\end{align*}
By Lemma~\ref{lem.gl wlog}, we may assume without loss of generality that  $Q=\operatorname{diag}(\mathbf{1}_p,-\mathbf{1}_q)$ for some $p\in[d]$ and $q=d-p$.
(Here and throughout, $\mathbf{1}_n$ denotes the all-ones vector in $\mathbb{R}^n$.)

We claim that $N_xM=\operatorname{span}\{Qx\}$ for every $x\in M$.
Defining $f(x):=x^\top Qx-1=\sum_{i}Q_{ii}x_i^2-1$, then $\nabla f(x)=2Qx$.
As such, $x\in M$ only if $x^\top Qx=1$, only if $x\neq 0$, only if $\nabla f(x)=2Qx\neq0$.
Our claim then follows from Lemma~\ref{lem.normal gradient}.

Next, we verify that $M$ spans $\mathbb{R}^d$.
There are two cases to consider.
If $Q$ is the identity matrix, then $M$ contains the identity basis $\{e_i\}_{i\in[d]}$ and therefore spans.
Otherwise, $Q=\operatorname{diag}(\mathbf{1}_p,-\mathbf{1}_q)$ for some $p\in[d-1]$ and $q=d-p$.
Considering
\[
(\sqrt{2}e_1+e_j)^\top Q(\sqrt{2}e_1+e_j)
=2\|e_1\|^2-\|e_j\|^2
=1
\]
for every $j>p$, it follows that $M$ contains $\{e_i\}_{i=1}^p\cup\{\sqrt{2}e_1+e_j\}_{j=p+1}^d$ and therefore spans.

Next, we verify that $L:=\{xx^\top:x\in M\}$ spans $\mathbb{R}_\mathrm{sym}^{d\times d}$.
To accomplish this, we select an arbitrary $A\in \mathbb{R}_\mathrm{sym}^{d\times d}$ that is orthogonal to $L$, and we show that $A=0$.
Since
\[
x^\top Ax
=\operatorname{tr}(x^\top Ax)
=\operatorname{tr}(Axx^\top)
=\langle A,xx^\top\rangle
=0
\]
for every $x\in M$, it follows that $M\subseteq\{x\in\mathbb{R}^d:x^\top Ax=0\}$.
Now fix $x\in M$ and select an arbitrary $y\in T_xM=\operatorname{span}\{Qx\}^\perp$ and differentiable $f\colon I\to M$ such that $f(0)=x$ and $f'(0)=y$.
Then $f(t)^\top Af(t)=0$ for all $t\in I$, and so the product rule gives
\[
f'(t)^\top A f(t)+f(t)^\top A f'(t)
=0.
\]
Evaluating at $t=0$ gives $y^\top Ax=0$, and so $y\in\operatorname{span}\{Ax\}^\perp$.
This establishes that $\operatorname{span}\{Qx\}^\perp\subseteq\operatorname{span}\{Ax\}^\perp$, meaning $\operatorname{span}\{Ax\}\subseteq\operatorname{span}\{Qx\}$, which in turn implies that there exists $\alpha\in\mathbb{R}$ such that $Ax=\alpha Qx$.
Considering
\[
0
=x^\top Ax
=\alpha x^\top Qx
=\alpha,
\]
it follows that $Ax=0$.
Since our choice for $x\in M$ was arbitrary, and furthermore, $M$ spans $\mathbb{R}^d$, we conclude that $A=0$, as desired.

Next, we establish that $\operatorname{Sym}(M)=\{A\in\operatorname{GL}(d):A^\top QA=Q\}$.
For ($\subseteq$), select any $A\in\operatorname{Sym}(M)$.
Then $A\in\operatorname{GL}(d)$ by definition.
Furthermore, for every $x\in M$, it holds that $Ax\in M$, and so
\[
\langle A^\top QA,xx^\top\rangle
=(Ax)^\top QAx
=1
=x^\top Qx
=\langle Q,xx^\top\rangle.
\]
Since $\{xx^\top:x\in M\}$ spans $\mathbb{R}_\mathrm{sym}^{d\times d}$, it follows that $A^\top QA=Q$.
For ($\supseteq$), suppose $A\in\operatorname{GL}(d)$ satisfies $A^\top QA=Q$.
Then for every $x\in\mathbb{R}^d$, it holds that
\[
(Ax)^\top QAx
=x^\top A^\top QAx
=x^\top Qx,
\]
meaning $AM=M$.
As such, $A\in\operatorname{Sym}(M)$, as desired.

In words, we have shown that $\operatorname{Sym}(M)$ is the group of linear transformations $A\in\operatorname{GL}(d)$ that leave invariant the symmetric bilinear form $(x,y)\mapsto x^\top Qy$.
This group is known as the (indefinite) orthogonal group $\operatorname{O}(p,q)$, where $\operatorname{O}(d,0):=\operatorname{O}(d)$.
We claim that the corresponding Lie algebra is given by
\[
\mathfrak{sym}(M)
=\{Z\in\mathbb{R}^{d\times d}:Z^\top=-QZQ\}
=\{[\begin{smallmatrix}
  A & B\\
  B^\top & D
\end{smallmatrix}]:A\in\mathbb{R}_\mathrm{antisym}^{p\times p},~D\in\mathbb{R}_\mathrm{antisym}^{q\times q}\}.
\]
(This is presumably well known, but the proof is short, so we include it.)
Select $Z\in\mathfrak{sym}(M)$ so that there exists a differentiable $f\colon I\to\operatorname{Sym}(M)$ such that $f(0)=\operatorname{id}$ and $f'(0)=Z$.
Differentiating the identity $f(t)^\top Qf(t)=Q$ gives
\[
f'(t)^\top Qf(t)+f(t)^\top Qf'(t)=0.
\]
Evaluating at $t=0$ then gives $Z^\top Q+QZ=0$, meaning $Z^\top=-QZQ$.
Writing $Z=[\begin{smallmatrix}A&B\\C&D\end{smallmatrix}]$ then reveals that
\[
[\begin{smallmatrix}A^\top&C^\top\\B^\top&D^\top\end{smallmatrix}]
=Z^\top
=-QZQ
=[\begin{smallmatrix}-A&B\\C&-D\end{smallmatrix}],
\]
from which it follows that $Z=[\begin{smallmatrix}A&B\\B^\top&D\end{smallmatrix}]$ with $A\in\mathbb{R}_\mathrm{antisym}^{p\times p}$ and $D\in\mathbb{R}_\mathrm{antisym}^{q\times q}$.
Furthermore, any such matrix satisfies
\[
[\begin{smallmatrix}A&B\\B^\top&D\end{smallmatrix}]^\top
=[\begin{smallmatrix}A^\top&B\\B^\top&D^\top\end{smallmatrix}]
=[\begin{smallmatrix}-A&B\\B^\top&-D\end{smallmatrix}]
=-Q[\begin{smallmatrix}A&B\\B^\top&D\end{smallmatrix}]Q.
\]
It remains to show that every $Z\in\mathbb{R}^{d\times d}$ satisfying $Z^\top=-QZQ$ necessarily resides in $\mathfrak{sym}(M)$.
To this end, define $f\colon\mathbb{R}\to\operatorname{GL}(d)$ by $f(t)=e^{tZ}$.
Then $f(0)=\operatorname{id}$ and $f'(0)=Z$.
It suffices to verify that $f(t)^\top Qf(t)=Q$, since this would imply $f\colon\mathbb{R}\to\operatorname{Sym}(M)$, meaning $Z\in\mathfrak{sym}(M)$.
Since $Q^{-1}=Q$, we have
\[
f(t)^\top Qf(t)
=e^{tZ^\top}Qe^{tZ}
=e^{-tQZQ}Qe^{tZ}
=e^{-tQZQ^{-1}}Qe^{tZ}
=Qe^{-tZ}Q^{-1}Qe^{tZ}
=Q.
\]

At this point, we may compute
\[
n^\star(M)
=\frac{\operatorname{codim}\mathfrak{sym}(M)}{\operatorname{codim}M}
=\frac{d^2-\binom{d}{2}}{d-(d-1)}
=\binom{d+1}{2}.
\]
Furthermore, Lemma~\ref{lem.normal space}(b) gives
\[
(S_xM)^\perp
=\{zx^\top:z\in N_xM\}
=\{zx^\top:z\in \operatorname{span}\{Qx\}\}
=\operatorname{span}\{Qxx^\top\}.
\]
Put $n:=\binom{d+1}{2}$.
Considering Lemma~\ref{lem.sample complexity}, we want to find $\{x_i\}_{i\in[n]}$ in $M$ for which the subspaces $\{\operatorname{span}\{Qx_ix_i^\top\}\}_{i\in[n]}$ are linearly independent.
This occurs precisely when $\{Qx_ix_i^\top\}_{i\in[n]}$ are linearly independent, which in turn occurs precisely when $\{x_ix_i^\top\}_{i\in[n]}$ are linearly independent.
Overall, we wish to find an open and dense subset $O$ of $M^n$ such that for every $\{x_i\}_{i\in[n]}\in O$, it holds that $\{x_ix_i^\top\}_{i\in[n]}$ are linearly independent.

First, let $O'$ denote the set of $\{x_i\}_{i\in[n]}\in (\mathbb{R}^d)^n$ such that $\{x_i x_i^\top\}_{i\in[n]}$ is linearly independent.
We claim that $O'$ is open and dense in $(\mathbb{R}^d)^n$.
Select an orthonormal basis $\{B_i\}_{i\in[n]}$ for $\mathbb{R}_\mathrm{sym}^{d\times d}$, consider the mapping $A\colon(\mathbb{R}^d)^n\to\mathbb{R}^{n\times n}$ defined by $(A(\{x_i\}_{i\in[n]}))_{jk}:=\langle B_j,x_kx_k^\top\rangle$, and let $p$ denote the polynomial in $dn$ variables defined by
\[
p(\{x_i\}_{i\in[n]}):=\operatorname{det}A(\{x_i\}_{i\in[n]}).
\]
Then $O'=p^{-1}(\mathbb{R}\setminus\{0\})$, which is open and dense in $(\mathbb{R}^d)^n$ provided $p$ is not the zero polynomial.
As such, it suffices to find $\{x_i\}_{i\in[n]}$ such that $p(\{x_i\}_{i\in[n]})\neq0$.
To this end, for each $i\in[n]$, consider the eigenvalue decomposition $B_i=\sum_{j\in[d]}\lambda_{ij}u_{ij}u_{ij}^\top$.
Then $\{u_{ij}u_{ij}^\top\}_{i\in[d],j\in[n]}$ is a spanning set for $\mathbb{R}_\mathrm{sym}^{d\times d}$, and so any choice of basis $\{u_{ij}u_{ij}^\top\}_{(i,j)\in S}$ in this spanning set has the desired property that $p(\{u_{ij}\}_{(i,j)\in S})\neq0$.

Finally, we claim that $O:=O'\cap M^n$ is open and dense in $M^n$.
Openness follows from the definition of the subspace topology.
To demonstrate denseness, consider the open set
\[
R
:=\{x\in\mathbb{R}^d:x^\top Qx>0\}.
\]
The mapping $g\colon R\to M$ defined by $g(x):=x/\sqrt{x^\top Qx}$ is surjective since $g(x)=x$ for every $x\in M\subseteq R$.
Moreover, for every $\{x_i\}_{i\in[n]}\in O'\cap R^m$, it holds that $\{x_ix_i^\top\}_{i\in[n]}$ is linearly independent, and so $\{g(x_i)g(x_i)^\top\}_{i\in[n]}$ is also linearly independent, meaning $\{g(x_i)\}_{i\in[n]}\in O$.
As such, the continuous function $h\colon R^n\to M^n$ defined by $h(\{x_i\}_{i\in[n]}):=\{g(x_i)\}_{i\in[n]}$ has the property that $h(O'\cap R^n)=O$.
Since $O'\cap R^n$ is dense in $R^n$, it follows that $O=h(O'\cap R^n)$ is dense in $M^n=h(R^n)$, as desired.
\end{proof}

Finally, we consider a family of quadrics that includes cones:

\begin{theorem}
\label{thm.zero}
Select any full rank $Q\in \mathbb{R}_\mathrm{sym}^{d\times d}$ and put $M:=\{x\in\mathbb{R}^d \setminus \{0\}:x^\top Qx=0\}$.
\begin{itemize}
\item[(a)]
$M$ is empty if and only if $Q$ is positive definite or negative definite.
\item[(b)]
Otherwise, if $d=2$, then $n^\star(M) = 2$ but $n^\circ(M) =\infty$.
\item[(c)]
Otherwise, $n^\circ(M) = n^\star(M) = \binom{d+1}{2}-1$.
\end{itemize}
\end{theorem}

For the previous results in this section, we computed $n^\circ(M)$ by selecting a certain open and dense subset of $M^n$.
To accomplish this, we first identified an appropriate polynomial over $(\mathbb{R}^{d})^n$, and then we argued that the complement of the zero set of this polynomial has an open and dense intersection with $M^n$.
An analogous argument for Theorem~\ref{thm.zero} appears to require more powerful tools from real algebraic geometry.
A \textbf{real algebraic set} $V\subseteq\mathbb{R}^k$ is the simultaneous zero set of a finite collection of polynomials in $\mathbb{R}[x_1,\ldots,x_k]$, which in turn determines the ideal $I(V)\subseteq\mathbb{R}[x_1,\ldots,x_k]$ of all polynomials that vanish on $V$.
The \textbf{dimension} of a real algebraic set $V$ is the dimension of the ring $\mathbb{R}[x_1,\ldots,x_k]/I(V)$.
This notion of dimension can be difficult to compute directly.
Similarly, the definition of \textbf{nonsingular point} of a real algebraic set is particularly technical; see Definition~3.3.9 of~\cite{BochnakCR:13}.
To simplify our interaction with these notions, we factor our analysis through the following proposition, the proof of which is contained in Section~5 of~\cite{CahillMS:17}.

\begin{proposition}
\label{prop.key ingredients from RAG}
Let $V\subseteq\mathbb{R}^k$ denote the simultaneous zero set of real polynomials $\{p_i\}_{i\in[k-d]}$.
\begin{itemize}
\item[(a)]
If $V$ is a $C^\infty$ submanifold of $\mathbb{R}^k$, then its dimension as a real algebraic set equals its dimension as a manifold.
\item[(b)]
If $V$ has dimension $d$ as a real algebraic set and the Jacobian of $z\mapsto\{p_i(z)\}_{i\in[k-d]}$ has rank $k-d$ at $x\in V$, then $x$ is a nonsingular point of $V$.
\item[(c)]
If the nonsingular points in $V$ form a dense and path-connected subset, then for every real algebraic set $U\subseteq\mathbb{R}^k$, the relative complement $V\setminus U$ is either empty or both open and dense in $V$.
\end{itemize}
\end{proposition}

\begin{proof}[Proof of Theorem~\ref{thm.zero}]
First, (a) is immediate.
In what follows, we assume $Q$ is neither positive definite nor negative definite.
As in the proof of Theorem~\ref{thm.quadric}, we may assume without loss of generality that  $Q=\operatorname{diag}(\mathbf{1}_p,-\mathbf{1}_q)$ for some $p\in[d]$ and $q=d-p$.
Moreover, since $Q$ is neither positive nor negative definite, we have $p, q \in [d - 1]$.
Also, it holds that $N_xM = \operatorname{span}\{Qx\}$ for every $x \in M$.

Next, we verify that $L := \{xx^\top : x \in M\}$ satisfies $\operatorname{span}(L) = \operatorname{span}\{Q\}^\perp$.
The definition of $M$ immediately gives $\operatorname{span}(L) \subseteq \operatorname{span}\{Q\}^\perp$.
For $(\supseteq)$, consider $[\begin{smallmatrix}A&B\\C&D\end{smallmatrix}] \in \operatorname{span}\{Q\}^\perp$ with $A \in \mathbb{R}^{p \times p}$.
Then $C = B^\top$, $A \in \mathbb{R}^{p \times p}_{\mathrm{sym}}$, $D \in \mathbb{R}^{q \times q}_{\mathrm{sym}}$, and $\operatorname{tr}(A) = \operatorname{tr}(D)$.
For $x \in \mathbb{R}^p$ and $y \in \mathbb{R}^q$, we define
\[
\mathcal{A}(x,y) := [\begin{smallmatrix} xx^\top & xy^\top \\ yx^\top & yy^\top\end{smallmatrix}],
\]
and we observe that $\mathcal{A}(x,y) \in L$ precisely when $\|x\| = \|y\| \neq 0$.  It follows that $\operatorname{span}(L)$ contains $[\begin{smallmatrix} 0 & B \\ B^\top & 0 \end{smallmatrix}]$, since we can pass to the singular value decomposition of $B$ and apply the fact that for $\|x\| = \|y\| = 1$,
\[
[\begin{smallmatrix} 0 & xy^\top \\ yx^\top & 0 \end{smallmatrix}] = \frac{\mathcal{A}(x,y) - \mathcal{A}(x,-y)}{2} \in \operatorname{span}(L).
\]
It remains to show that $\operatorname{span}(L)$ contains $[\begin{smallmatrix} A & 0 \\ 0 & D \end{smallmatrix}]$.
To this end, it suffices to show that $[\begin{smallmatrix} X & 0 \\ 0 & Y \end{smallmatrix}]\in\operatorname{span}(L)$ for every $X,Y\succ0$ such that $\operatorname{tr}(X)=\operatorname{tr}(Y)$.
Indeed, we may express $[\begin{smallmatrix} A & 0 \\ 0 & D \end{smallmatrix}]$ as a difference of such matrices:
\[
[\begin{smallmatrix} A & 0 \\ 0 & D \end{smallmatrix}]
=\left[\begin{smallmatrix} A+\frac{c}{p} \operatorname{id}_p & 0 \\ 0 & D+\frac{c}{q} \operatorname{id}_q \end{smallmatrix}\right] - \left[\begin{smallmatrix} \frac{c}{p} \operatorname{id}_p & 0 \\ 0 & \frac{c}{q} \operatorname{id}_q \end{smallmatrix}\right],
\]
where $c>0$ is appropriately large.
Given eigenvalue decompositions $X = \sum_{i \in [p]} \lambda_i x_ix_i^\top$ and $Y = \sum_{j \in [q]} \mu_j y_jy_j^\top$, select measurable partitions
\[
I_1\sqcup\cdots\sqcup I_p
=[0,\operatorname{tr}(X)]
=[0,\operatorname{tr}(Y)]
=J_1\sqcup\cdots\sqcup J_q
\]
such that $|I_i|=\lambda_i$ and $|J_j|=\mu_j$. 
Then
\[
[\begin{smallmatrix} X & 0 \\ 0 & Y \end{smallmatrix}] 
= \sum_{i \in [p]}\sum_{j \in [q]} |I_i \cap J_j| \cdot [\begin{smallmatrix} x_ix_i^\top & 0 \\ 0 & y_jy_j^\top \end{smallmatrix}].
\]
Each term above resides in $\operatorname{span}(L)$ since for $\|x\| = \|y\| = 1$, it holds that
\[
[\begin{smallmatrix} xx^\top & 0 \\ 0 & yy^\top\end{smallmatrix}] = \frac{\mathcal{A}(x,y) + \mathcal{A}(x,-y)}{2} \in \operatorname{span}(L).
\]
As such, $[\begin{smallmatrix} X & 0 \\ 0 & Y \end{smallmatrix}]\in\operatorname{span}(L)$, as desired.

Next, we verify that $\operatorname{Sym}(M)=\{A\in\operatorname{GL}(d):A^\top QA=\alpha Q,~\alpha\neq0\}$.
For $(\subseteq)$, let $A \in \operatorname{Sym}(M)$.  Then
\[
\langle A^\top Q A, xx^\top \rangle
= (Ax)^\top Q Ax
= 0
= x^\top Q x
= \langle Q, xx^\top \rangle.
\]
Since $\operatorname{span}(L) = \operatorname{span}\{Q\}^\perp$, we have shown $A^\top Q A = \alpha Q$ for $\alpha \neq 0$.
For $(\supseteq)$, suppose $A \in \operatorname{GL}(d)$ satisfies $A^\top Q A = \alpha Q$ for some $\alpha \neq 0$.  Then for every $x \in \mathbb{R}^d$,
\[
(Ax)^\top Q (Ax) = x^\top A^\top Q A x = \alpha x^\top Q x,
\]
and so $AM = M$.  As such, $A \in \operatorname{Sym}(M)$ as desired.

We claim that the corresponding Lie algebra is given by
\begin{align*}
\mathfrak{sym}(M)
&=\{Z\in\mathbb{R}^{d\times d}:Z^\top=-QZQ + 2\beta \operatorname{id},~\beta \in \mathbb{R}\}\\
&=\{[\begin{smallmatrix}
  A_0 & B\\
  B^\top & D_0
\end{smallmatrix}] + \beta \operatorname{id} :A_0\in\mathbb{R}_\mathrm{antisym}^{p\times p},~D_0\in\mathbb{R}_\mathrm{antisym}^{q\times q},~\beta\in\mathbb{R}\}.
\end{align*}
Select $Z\in\mathfrak{sym}(M)$ so that there exists a differentiable $f\colon I\to\operatorname{Sym}(M)$ such that $f(0)=\operatorname{id}$ and $f'(0)=Z$.
Define $\alpha \colon I \rightarrow \mathbb{R} \setminus \{0\}$ by
\[
\alpha(t) := \frac{\operatorname{tr}(f(t)^\top Q f(t) Q)}{\|Q\|_F^2}.
\]
Differentiating the identity $f(t)^\top Qf(t)=\alpha(t)Q$ gives
\[
f'(t)^\top Qf(t)+f(t)^\top Qf'(t)=\alpha'(0).
\]
Evaluating at $t=0$ then gives $Z^\top Q+QZ=\alpha'(0)Q$, meaning $Z^\top=-QZQ + \alpha'(0)\operatorname{id}$.
Writing $Z=[\begin{smallmatrix}A&B\\C&D\end{smallmatrix}]$ then reveals that
\[
[\begin{smallmatrix}A^\top&C^\top\\B^\top&D^\top\end{smallmatrix}]
=Z^\top
=-QZQ + \alpha'(0)\operatorname{id}
=[\begin{smallmatrix}\alpha'(0)\operatorname{id}_p-A&B\\C&\alpha'(0)\operatorname{id}_q-D\end{smallmatrix}],
\]
from which it follows that $C = B^T$, $A + A^T = \alpha'(0) \operatorname{id}_p$, and $D + D^T = \alpha'(0) \operatorname{id}_q$.  
Setting $\beta = \alpha'(0)/2$, we see that $Z=[\begin{smallmatrix}A_0 &B\\B^\top&D_0\end{smallmatrix}] + \beta \operatorname{id}$ for $A_0\in\mathbb{R}_\mathrm{antisym}^{p\times p}$, $D_0\in\mathbb{R}_\mathrm{antisym}^{q\times q}$, and $\beta \in \mathbb{R}$.
Furthermore, any such matrix satisfies
\[
([\begin{smallmatrix}A_0 &B\\B^\top&D_0\end{smallmatrix}] + \beta \operatorname{id})^\top
=[\begin{smallmatrix}A_0^\top&B\\B^\top&D_0^\top\end{smallmatrix}] + \beta\operatorname{id}
=[\begin{smallmatrix}-A_0&B\\B^\top&-D_0\end{smallmatrix}] + \beta\operatorname{id}
=-Q([\begin{smallmatrix}A_0&B\\B^\top&D_0\end{smallmatrix}] + \beta\operatorname{id})Q + 2\beta\operatorname{id}.
\]
Fixing $\beta \in \mathbb{R}$, it remains to show that every $Z\in\mathbb{R}^{d\times d}$ satisfying $Z^\top=-QZQ + 2\beta\operatorname{id}$ necessarily resides in $\mathfrak{sym}(M)$.
To this end, define $f\colon\mathbb{R}\to\operatorname{GL}(d)$ by $f(t)=e^{tZ}$.
Then $f(0)=\operatorname{id}$ and $f'(0)=Z$.
It suffices to verify that $f(t)^\top Qf(t)=\alpha Q$ for some $\alpha \neq 0$, since this would imply $f\colon\mathbb{R}\to\operatorname{Sym}(M)$, meaning $Z\in\mathfrak{sym}(M)$.
Since $Q^{-1}=Q$, we have
\begin{align*}
f(t)^\top Qf(t)
=e^{tZ^\top}Qe^{tZ}
&=e^{-tQZQ + 2\beta\operatorname{id}}Qe^{tZ}\\
&=e^{2\beta} e^{-tQZQ^{-1}}Qe^{tZ}
=e^{2\beta} Qe^{-tZ}Q^{-1}Qe^{tZ}
=e^{2\beta} Q.
\end{align*}

At this point, we may compute
\[
n^\star(M)
=\frac{\operatorname{codim}\mathfrak{sym}(M)}{\operatorname{codim}M}
=\frac{d^2-(\binom{d}{2} + 1)}{d-(d-1)}
=\binom{d+1}{2} - 1.
\]
It remains to compute $n^\circ(M)$.
For this, (b) and (c) require different approaches.

For (b), we take $d = 2$.
Express $M = M_+ \sqcup M_-$, where
\[
M_{\pm} := \{x \in \mathbb{R}^2 \setminus \{0\} : x_1 = \pm x_2\}.
\]
Fix $n \in \mathbb{N}$ and select $O \in \mathcal{O}(M^n)$.
Observe that $(1,1) \in M_+ \subseteq M$, and therefore $\{(1,1)\}_{i \in [n]} \in M^n$.
Since $O$ is dense in $M^n$, there exists $\{x_i\}_{i \in [n]} \in O$ arbitrarily close to $\{(1,1)\}_{i \in [n]}$; i.e., $x_i \in M_+$ for all $i$.
Since $T_{x_i} M = T_{x_i} M_+$ for every $i \in [n]$, it holds that
\[
\bigcap_{i \in [n]} S_{x_i} M = \bigcap_{i \in [n]} S_{x_i} M_+ \supseteq \mathfrak{sym}(M_+),
\]
where the last step follows from Lemma~\ref{lem.sample complexity}.  
As in the proof of Theorem~\ref{thm.subspace}, $\operatorname{codim} \mathfrak{sym}(M_+) = 1$.  
On the other hand, $\operatorname{codim} \mathfrak{sym}(M) = 2$, and so $\bigcap_{i \in [n]} S_{x_i} M \neq \mathfrak{sym}(M)$.  
All together, we conclude $n^\circ(M) = \infty$.

For (c), we consider $d \geq 3$.  Lemma~\ref{lem.normal space}(b) gives
\[
(S_xM)^\perp
=\{zx^\top:z\in N_xM\}
=\{zx^\top:z\in \operatorname{span}\{Qx\}\}
=\operatorname{span}\{Qxx^\top\}.
\]
Put $n:=\binom{d+1}{2} - 1$.
Considering Lemma~\ref{lem.sample complexity}, we want to find $\{x_i\}_{i\in[n]}$ in $M$ for which the subspaces $\{\operatorname{span}\{Qx_ix_i^\top\}\}_{i\in[n]}$ are linearly independent.
This occurs precisely when $\{Qx_ix_i^\top\}_{i\in[n]}$ are linearly independent, which in turn occurs precisely when $\{x_ix_i^\top\}_{i\in[n]}$ are linearly independent.
Overall, we wish to find an open and dense subset $O$ of $M^n$ such that for every $\{x_i\}_{i\in[n]}\in O$, it holds that $\{x_ix_i^\top\}_{i\in[n]}$ are linearly independent.

As in Theorem~\ref{thm.quadric}, let $O'$ denote the set of $\{x_i\}_{i\in[n]}\in (\mathbb{R}^d)^n$ such that $\{x_i x_i^\top\}_{i\in[n]}$ is linearly independent; recall that $O'$ is open and dense in $(\mathbb{R}^d)^n$.  Then $O := O' \cap M^n$ is open in $M^n$, and it remains to show that $O$ is dense in $M^n$.  By scaling, it suffices to show that
\[
Z = \{(x_1, \ldots, x_n) \in O : \|x_i\| = \sqrt{2} \text{ for all } i \in [n]\}
\]
is dense in 
\[
V = \{(x_1, \ldots, x_n) \in M^n : \|x_i\| = \sqrt{2} \text{ for all } i \in [n]\}.
\]
Considering $Z$ takes the form $Z=V\setminus U$ for some real algebraic set $U$, we are in a position to apply Proposition~\ref{prop.key ingredients from RAG}(c).
We proceed in two cases.

\medskip

\textbf{Case I.}
$p \neq 1 \neq q$.
By Proposition~\ref{prop.key ingredients from RAG}(c), it suffices to show that
\begin{itemize}
\item[(i)] $Z$ is nonempty,
\item[(ii)] every point in $V$ is nonsingular, and
\item[(iii)] $V$ is connected.
\end{itemize}
For (i), recall that $\operatorname{span}(L) = \operatorname{span}\{Q\}^\perp$ and choose an appropriately scaled basis $\{x_ix_i^\top\}_{i \in [n]}$ for $\operatorname{span}(L)$ so that $\{x_i\}_{i\in[n]} \in Z$.
Next, we demonstrate (ii).
We may identify $V$ with $(S^{p - 1} \times S^{q - 1})^n$, meaning $V$ is a $C^\infty$ submanifold of $\mathbb{R}^{dn}$, and so Proposition~\ref{prop.key ingredients from RAG}(a) gives that its dimension as a real algebraic set equals $dn-2n$.
For each $x_i \in \mathbb{R}^d$, write $x_i = (y_i,z_i)$ with $y_i \in \mathbb{R}^p$ and $z_i \in \mathbb{R}^q$, and define the polynomials
\[
P_i(x_1, \ldots, x_n) = \|y_i\|^2 - 1,
\qquad
Q_i(x_1, \ldots, x_n) = \|z_i\|^2 - 1.
\]
Observe that $V$ is the simultaneous zero set of $\{P_i,Q_i\}_{i\in[n]}$.
For $\{x_i\}_{i\in[n]} \in V$, the Jacobian
\[
J(x_1, \ldots, x_n) := \operatorname{blockdiag}(2y_1^\top, 2z_1^\top, \ldots, 2y_n^\top, 2z_n^\top)
\]
satisfies $JJ^\top = 4\operatorname{id}$, and so $J$ has rank $2n$.
Proposition~\ref{prop.key ingredients from RAG}(b) then gives (ii).  
Finally, the identification $V=(S^{p - 1} \times S^{q - 1})^n$ implies (iii), as desired.

\medskip

\textbf{Case II.}
Either $p = 1$ or $q = 1$.
We may identify $V$ with $(\{\pm 1\} \times S^{d - 2})^n$, and so we may express $V$ as the union of $2^n$  connected components
\[
V
= \bigsqcup_{\epsilon \in \{\pm1\}^n} V_\epsilon, 
\qquad 
V_\epsilon 
:= \{([\begin{smallmatrix} \epsilon_1 \\  z_1 \end{smallmatrix}], \ldots, [\begin{smallmatrix} \epsilon_n \\ z_n \end{smallmatrix}]) : z_1, \ldots, z_n \in S^{d - 2}\}.
\]
We will apply Proposition~\ref{prop.key ingredients from RAG}(c) to show that, for each $\epsilon \in \{\pm1\}^n$, the intersection $V_\epsilon \cap Z$ is dense in $V_\epsilon$.
By the argument in Case~I, $Z$ is nonempty, and so we may select $\{x_i\}_{i\in[n]} \in Z$.
Since $Z \subseteq V$, there exists $\epsilon \in \{\pm 1\}^n$ such that $\{x_i\}_{i\in[n]} \in V_{\epsilon}$.
Then for each $\epsilon' \in \{\pm 1\}^n$, it holds that $\{\epsilon'_i \epsilon_i x_i\}_{i\in[n]} \in V_{\epsilon'} \cap Z$.
As such, $V_{\epsilon} \cap Z$ is nonempty for every $\epsilon \in \{\pm 1\}^n$.
Moreover, $V_\epsilon$ is connected and, by the argument in Case~I, every point in $V_\epsilon$ is nonsingular.
All together, we conclude that each $V_\epsilon \cap Z$ is dense in $V_\epsilon$, and so $Z$ is dense in $V$, as desired.
\end{proof}

\section{Application to density estimation}

In this section, we apply Lie PCA to perform density estimation in various settings.
For each experiment, we consider a manifold with a nontrivial Lie group.
We draw points $\{x_i\}_{i\in[n]}$ in $\mathbb{R}^d$ according to some distribution supported on that manifold.
The density estimation algorithm then uses these points to produce $N\gg n$ draws $\{y_s\}_{s\in[N]}$ from an estimate of the underlying distribution.
For these experiments, we grant access to the dimension $r$ of the manifold and the dimension $\ell$ of the Lie algebra so as to isolate the performance of each algorithm from the task of learning hyperparameters.

Our algorithm first runs local PCA on $k$ nearest neighbors from each of the data points to produce estimates $\{T_i\}_{i\in[n]}$ of the tangent spaces at each of the sample points.
We then run Algorithm~\ref{alg.LiePCA} to obtain an estimate $\operatorname{span}\{v_j\}_{j\in[\ell]}$ of $\mathfrak{sym}(M)$.
Then for each $s\in[N]$, we draw $t\sim\mathsf{Unif}([n])$ and a random $A$ from $\operatorname{span}\{v_j\}_{j\in[\ell]}$ with spherical Gaussian distribution, and then we put $y_s:=e^Ax_t$.
If $y_s$ is too far away from $\{x_i\}_{i\in[n]}$, then we replace $y_s$ with another draw from this random process, repeating as necessary.

We compare our approach to a few alternatives.
One baseline is to simply draw each $y_s$ uniformly from $\{x_i\}_{i\in[n]}$.
We denote this by \textbf{BL1}.
Another baseline, which we call \textbf{BL2}, is to draw each $y_s$ from the unknown distribution.
While this is not a plausible alternative, it indicates how well an algorithm can possibly perform.
We also consider a standard approach known as kernel density estimation, which we denote by \textbf{KDE}.
For this approach, we draw $t\sim\mathsf{Unif}([n])$ and a random Gaussian vector $g$ with covariance determined by Silverman's rule of thumb~\cite{Silverman:86}, and then put $y_s:=x_t+g$.
This sort of estimate is specifically designed for the regime in which $n$ is large, where the covariance decays gracefully to zero.
(We will find that $n\ll100$ is not large enough for this method to perform well.)
Finally, we consider local PCA, which we denote by \textbf{LPCA}, in which the same estimates $\{T_i\}_{i\in[n]}$ obtained for the Lie PCA approach are used.
Here, we draw $t\sim\mathsf{Unif}([n])$ and a random Gaussian vector $g$ in $T_t$, and then put $y_s:=x_t+g$.
If $y_s$ is too far away from $\{x_i\}_{i\in[n]}$, then we replace $y_s$ with another draw from this random process, repeating as necessary.

We consider two different metrics for measuring the performance of these algorithms.
Both metrics compare $\{y_s\}_{s\in[N]}$ to a fresh draw $\{z_s\}_{s\in[N]}$ from the underlying distribution.
Taking inspiration from~\cite{ArjovskyCB:17}, our first metric takes the (normalized) earth mover's distance between $\{y_s\}_{s\in[N]}$ and $\{z_s\}_{s\in[N]}$.
Conveniently, this distance can be obtained by linear programming.
Indeed, defining $D\in\mathbb{R}^{N\times N}$ by $D_{st}:=\|y_s-z_t\|_2$ gives
\[
\operatorname{nEMD}(\{y_s\}_{s\in[N]},\{z_s\}_{s\in[N]})
=\min\{\tfrac{1}{N}\operatorname{tr}(DX):X\in\mathbb{R}^{N\times N},X1=X^\top1=1,X\geq0\}.
\]
Intuitively, the normalized earth mover's distance captures the average distance traveled per point by optimal transport from $\{y_s\}_{s\in[N]}$ to $\{z_s\}_{s\in[N]}$.
We found that for $N=300$, this distance can computed in CVX~\cite{GrantB:online} in about 15 seconds.
We also wanted a metric that is determined by the underlying supports of the densities rather than being sensitive to fluctuations in the densities.
This led us to also consider the Hausdorff distance between $\{y_s\}_{s\in[N]}$ and $\{z_s\}_{s\in[N]}$, which is much faster to compute:
\[
\operatorname{Hausdorff}(\{y_s\}_{s\in[N]},\{z_s\}_{s\in[N]})
=\max\Big(\max_{s\in[N]}\min_{t\in[N]}\|x_s-y_t\|_2,
\max_{t\in[N]}\min_{s\in[N]}\|x_s-y_t\|_2
\Big).
\]
In words, if we identify the closest member of $\{y_s\}_{s\in[N]}$ to each point in $\{z_s\}_{s\in[N]}$, and vice versa, then the Hausdorff distance reports the largest of these distances.

\begin{table}[t]
\caption{Density estimation error in normalized earth mover's and Hausdorff distances\label{table.table}}
\begin{center}
\begin{tabular}{p{3cm}|p{2cm}p{2cm}p{2cm}p{2cm}|p{2cm}}\hline
Manifold & BL1 & KDE & LPCA & Lie PCA & BL2 \\\hline\hline
line 
& \textbf{0.3901}   &  0.7391 &   0.4902  &  0.4185 & 0.1543 \\
& 1.0886   & 2.1822  &  \textbf{0.5558}  &  0.5838 & 0.3645 \\\hline
ellipse 
&0.4782   & 0.8620   & 0.5346  &  \textbf{0.4686} & 0.1909\\
& 0.8925  &  2.2269  &  1.3633  & \textbf{0.1569} & 0.1390 \\\hline
hyperbola 
&0.3322  &  0.5645  &  0.4337  &  \textbf{0.2750} & 0.1444 \\
&3.0194  &  2.4414  &  2.5090  &  \textbf{2.3396} & 0.4246 \\\hline
ellipse + noise
&0.4030    &0.4959    &0.3580    &\textbf{0.3057}    &0.3365\\
&0.7951    &1.8170    &1.0371    &\textbf{0.5512}    &0.6492\\\hline
torus
&1.0040    &1.0465    &0.9549    &\textbf{0.7022}    &0.5469\\
&\textbf{1.5197}    &3.2996    &2.5790    &1.7882    &0.9122\\\hline
\end{tabular}
\end{center}
\end{table}

\begin{figure}
\begin{center}
\includegraphics[width=\textwidth,trim={10 100 0 80},clip]{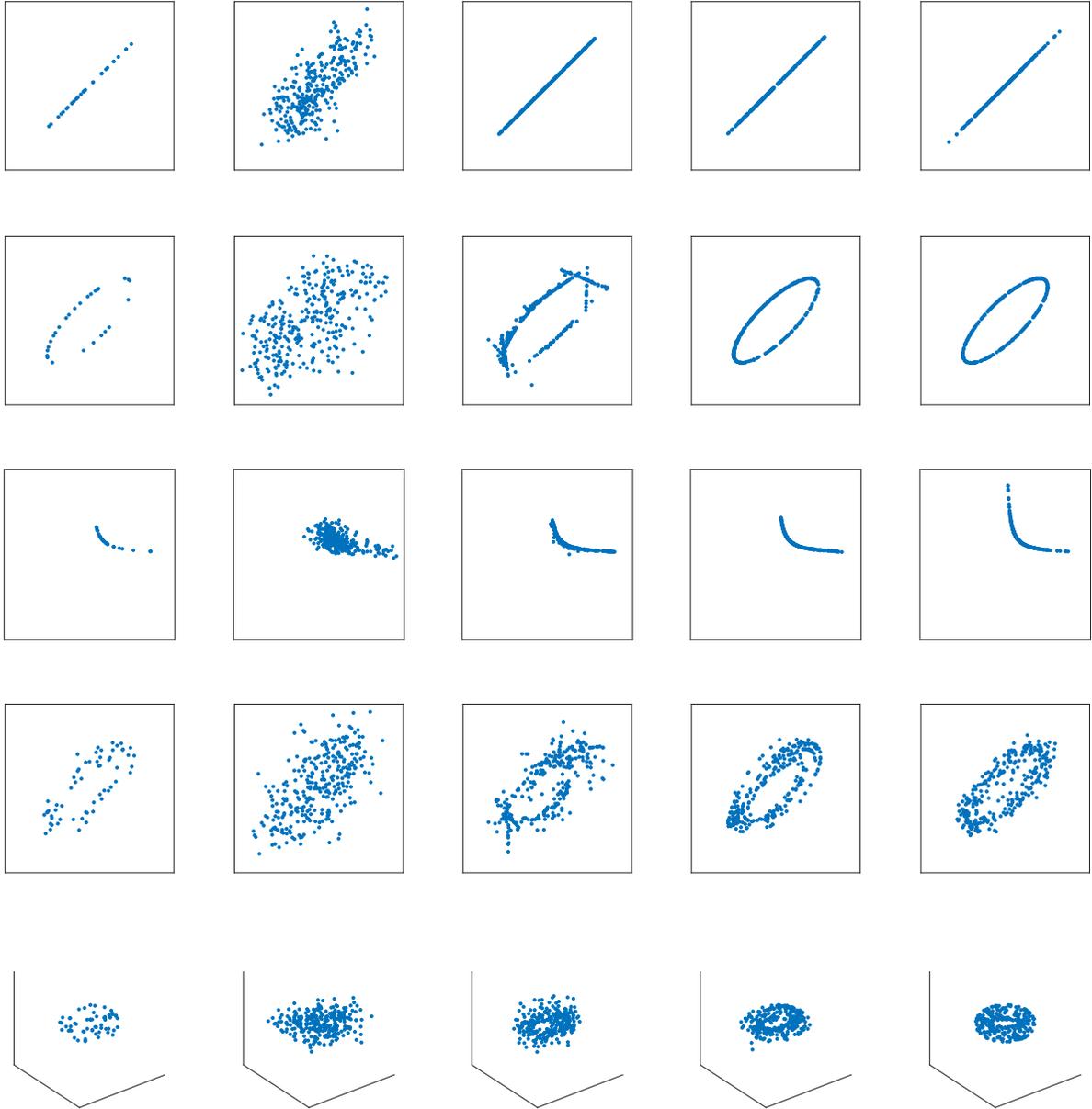}
\end{center}
\caption{\label{fig.comparison_2}
Illustration of results from Table~\ref{table.table}.
Given $n$ independent draws from an unknown distribution in $\mathbb{R}^d$, simulate an additional $N\gg n$ draws.
\textbf{(left)}
Given data points.
The origin of $\mathbb{R}^d$ is located at the center of the display box.
\textbf{(middle left)}
Simulated draws using kernel density estimation with Silverman's rule of thumb.
\textbf{(middle)}
Simulated draws using local PCA.
\textbf{(middle right)}
Simulated draws using Lie PCA.
\textbf{(right)}
Fresh draws from the true distribution.
}
\end{figure}

We considered densities on three different manifolds in $\mathbb{R}^2$.
For these instances, we take $n=30$, $N=300$, $k=2$, $r=1$, and $\ell=1$.
We then considered a noisy version of a manifold-supported density in $\mathbb{R}^2$, in which we take $n=60$, $N=300$, $k=10$, $r=1$, and $\ell=1$.
Finally, we considered a density on a manifold in $\mathbb{R}^3$, in which we take $n=60$, $N=300$, $k=20$, $r=2$, and $\ell=1$.
Results are reported in Table~\ref{table.table} and illustrated in Figure~\ref{fig.comparison_2}.

\section{Discussion}

This paper introduced a spectral method that uses the output from local PCA to estimate the Lie algebra corresponding to the symmetry group of the underlying manifold.
In this section, we point out a few opportunities for future work.
First, recall that our spectral method arises from relaxing the set of $\ell$-dimensional Lie algebras to the Grassmannian.
It would be interesting to somehow round the solution of the spectral method to a nearby Lie algebra.
Next, our sample complexity results came from focusing on specific families of manifolds.
It would be nice to have more general results in this vein.
For example, can we characterize the manifolds $M$ for which $n^\circ(M)=n^\star(M)$?
When applying Lie PCA to the density estimation problem, we would prefer a principled approach for drawing $A$ from our estimate of $\mathfrak{sym}(M)$; we currently apply an ad hoc adaptation of Silverman's rule of thumb.
Finally, the performance of Lie PCA for density estimation appears to depend on whether the symmetry group acts transitively on the manifold.
For example, the torus partitions into circular orbits under the action of its symmetry group.
For this manifold, Lie PCA will encourage motion along these circles without regard for the other dimension of the manifold.
However, local PCA captures some information about this other dimension, and it would be interesting to somehow incorporate this into the density estimation algorithm.

\section*{Acknowledgments}

DGM was partially supported by AFOSR FA9550-18-1-0107 and
NSF DMS 1829955.
HP was partially supported by an AMS-Simons Travel Grant.


\begin{thebibliography}{WW}

\bibitem{ArjovskyCB:17}
M.\ Arjovsky, S.\ Chintala, L.\ Bottou,
Wasserstein Generative Adversarial Networks,
ICML 2017, 214--223.

\bibitem{BochnakCR:13}
J.\ Bochnak, M.\ Coste, M.\ F.\ Roy,
Real Algebraic Geometry,
Vol.\ 36, Springer Science \& Business Media, 2013.

\bibitem{BrunaM:13}
J.\ Bruna, S.\ Mallat,
Invariant scattering convolution networks,
IEEE Trans.\ Pattern Anal.\ Mach.\ Intell.\ 35 (2013) 1872--1886.

\bibitem{CahillCC:19}
J.\ Cahill, A.\ Contreras, A.\ Contreras-Hip,
Classifying Signals Under a Finite Abelian Group Action:\ The Finite Dimensional Setting,
arXiv:1911.05862

\bibitem{CahillCC:20}
J.\ Cahill, A.\ Contreras, A.\ Contreras-Hip,
Complete set of translation invariant measurements with Lipschitz bounds,
Appl.\ Comput.\ Harmon.\ Anal.\ 49 (2020) 521--539.

\bibitem{CahillMS:17}
J.\ Cahill, D.\ G.\ Mixon, N.\ Strawn,
Connectivity and irreducibility of algebraic varieties of finite unit norm tight frames,
SIAM J.\ Appl.\ Algebra Geometry 1 (2017) 38--72.

\bibitem{ClumMS:20}
C.\ Clum, D.\ G.\ Mixon, T.\ Scarnati,
Matching Component Analysis for Transfer Learning,
SIAM J.\ Math.\ Data Sci.\ 2 (2020) 309--334.

\bibitem{DumitrascuVME:19}
B.\ Dumitrascu, S.\ Villar, D.\ G.\ Mixon, B.\ E.\ Engelhardt,
Optimal marker gene selection for cell type discrimination in single cell analyses,
BioRxiv (2019) 599654.

\bibitem{GrantB:online}
M.\ Grant, S.\ Boyd,
CVX:\ Matlab software for disciplined convex programming,
\url{http://cvxr.com/cvx}

\bibitem{Hall:15}
B.\ Hall,
Lie groups, Lie algebras, and representations:\ An elementary introduction,
Vol.\ 222. Springer, 2015.

\bibitem{KrizhevskySH:12}
A.\ Krizhevsky, I.\ Sutskever, G.\ E.\ Hinton,
Imagenet classification with deep convolutional neural networks,
NIPS 2012, 1097--1105.

\bibitem{Mallat:12}
S.\ Mallat,
Group invariant scattering,
Comm.\ Pure Appl.\ Math.\ 65 (2012) 1331--1398.

\bibitem{McWhirterMV:20}
C.\ McWhirter, D.\ G.\ Mixon, S.\ Villar,
Squeezefit:\ Label-aware dimensionality reduction by semidefinite programming,
IEEE Trans.\ Inform.\ Theory 66 (2020) 3878--3892.

\bibitem{Silverman:86}
B.\ W.\ Silverman, 
Density Estimation for Statistics and Data Analysis,
Chapman \& Hall/CRC, 1986.

\bibitem{SimardSP:03}
P.\ Y.\ Simard, D.\ Steinkraus, J.\ C.\ Platt,
Best Practices for Convolutional Neural Networks Applied to Visual Document Analysis,
ICDAR 2003, 958.

\bibitem{SzegedyEtal:15}
C.\ Szegedy, W.\ Liu, Y.\ Jia, P.\ Sermanet, S.\ Reed, D.\ Anguelov, D.\ Erhan, V.\ Vanhoucke, A.\ Rabinovich,
Going deeper with convolutions,
CVPR 2015, 1--9.

\end{thebibliography}
\end{document}